\documentclass{article}

\usepackage[square,numbers]{natbib}
\usepackage[final]{neurips_2022}

\usepackage[utf8]{inputenc} 
\usepackage[T1]{fontenc}    
\usepackage{hyperref}       
\usepackage{url}            
\usepackage{booktabs}       
\usepackage{amsfonts}       
\usepackage{nicefrac}       
\usepackage{microtype}      
\usepackage{xcolor}         
\usepackage{algorithmic}
\usepackage{algorithm}
\usepackage{amsmath}
\usepackage{amssymb}
\usepackage{amsthm}
\usepackage{pifont}
\usepackage{graphicx}
\newcommand{\cmark}{\ding{51}}
\newcommand{\xmark}{\ding{55}}

\newtheorem{Def}{Definition}
\newtheorem{theorem}{Theorem}

\newtheorem*{theorem*}{Theorem}
\newtheorem*{lemma*}{Lemma}
\newtheorem*{proof*}{Proof}

\title{Guided Diffusion Model for Adversarial \\
Purification from Random Noise}

\author{%
  Quanlin Wu \\
  Yuanpei College\\
  Peking University\\
  \texttt{quanlin@pku.edu.cn} \\
  \And
  Hang Ye \\
  Yuanpei College\\
  Peking University\\
  \texttt{yehang@pku.edu.cn} \\
  \And
  Yuntian Gu \\
  Yuanpei College\\
  Peking University\\
  \texttt{guyuntian@stu.pku.edu.cn} \\
}

\begin{document}

\maketitle
\begin{abstract}
In this project, we propose a novel guided diffusion purification approach to provide a strong defense against adversarial attacks. Our model achieves $89.62\%$ robust accuracy under PGD-$\ell_\infty$ attack ($\epsilon=8/255$) on the CIFAR-10 dataset. We first explore the essential correlations between unguided diffusion models and randomized smoothing, enabling us to apply the models to certified robustness. The empirical results show that our models outperform randomized smoothing by $5\%$ when the certified $\ell_2$ radius $r$ is larger than  $0.5$. 
\end{abstract}

\section{Introduction}
While deep neural networks have demonstrated remarkable capabilities in complicated tasks, it has been noticed by the community that they are vulnerable to adversarial attacks \cite{goodfellow2014explaining, szegedy}. Specifically, altering images with slight perturbations that are imperceptible to humans, can mislead trained neural networks to unexpected predictions, which poses a security threat in real-world scenarios. Various kinds of defense strategies have been proposed to protect DNN-based classifiers from adversarial attacks. Among them, \emph{adversarial training} \cite{pgd} has become a widely-used defense form. Despite their effectiveness, most adversarial training methods fail against suitably powerful and unseen attacks. In addition, it often requires higher computational complexity for training.

Another class of defense methods, often termed with \emph{adversarial purification}, aimed at using the generative models to recover corrupted examples in a pre-processing manner. Recently, denoising diffusion probabilistic models \cite{ddpm} have shown impressive performance as powerful generative models, beating GANs in image generation \cite{beatGAN}. Both works \cite{guide, diff_puri} propose a novel purification approach based on diffusion models. In sharp contrast with prevalent generative models, including GANs \cite{gan} and VAEs \cite{vae}, diffusion models define two processes: (i) a forward diffusion process that converts inputs into Gaussian noises step by step, and (ii) a reverse generative process recovering clean images from noises. Leveraging diffusion models for adversarial purification is a natural fit, as the adversarial perturbations will be gradually smoothed and dominated by the injected Gaussian noises. 
In addition, the stochasticity of diffusion models provides a powerful defense against adversarial attacks.

\textbf{We summarize our main contributions as follows:}
\begin{itemize}
\item Inspired by recent works, we propose a novel guided diffusion-based approach to purify adversarial images. Empirical results demonstrate the effectiveness of our method.
\item  We reveal the fundamental correlations between the unguided diffusion-model-based adversarial purification and randomized smoothing, enabling a provable defense mechanism.
\item  We present a theoretical analysis of the diffusion process in adversarial purification, which can provide some insights into its properties.
\item 
We also include denoising diffusion implicit models (DDIM) and study the trade-off between the inference speed and robustness.
\end{itemize}

\section{Backgrounds}

\subsection{Adversarial Robustness}

\textbf{Adversarial Training.} Adversarial training is an empirical defense method, which involves adversarial examples during the training of neural networks. Let $g_\phi: \mathcal{X} \rightarrow \mathcal{Y}$ denote the target classifier. Worst-case risk minimization can be formulated as the following saddle-point problem
\begin{equation}
\min_{\phi}\mathbb{E}_{(x, y) \sim p_{\text{data}}} [\max_{x' \in B(x)} \mathcal{L}(g_\phi(x'),y)]
\end{equation}
where $\mathcal{L}$ is a loss function and $\mathcal{B}(x)$ is the set of allowed perturbations in the neighborhood of $x$.

\textbf{Adversarial Purification.} Recently, leveraging generative models to purify the images from adversarial perturbations before classification has become a promising counterpart of adversarial training. In this manner, neither assumptions on the form of attacks nor the architecture details of classifiers are required. The generative models for purification can be trained independently and paired with standard classifiers, which is less time-consuming compared with adversarial training. We just need to slightly modify the formula above to represent adversarial purification
\begin{equation}
\min_{\phi, \theta}\mathbb{E}_{(x, y) \sim p_{\text{data}}} [\max_{x' \in B(x)} \mathcal{L}(g_\phi(f_\theta(x'),y)]
\end{equation}
where $f_\theta$ denotes the generative models as a pre-processor. \cite{defensegan} propose defense-GAN, \cite{pixel} utilize pre-trained PixelCNN for purification and \cite{yoon2021adversarial} rely on denoising score-based model to remove adversarial perturbations. More recently, diffusion models have been applied to adversarial purification \cite{diff_puri, guide}. We base our works on \cite{diff_puri, guide} and make attempts to improve the efficiency of this defense mechanism. Furthermore, our models can be generalized to provide certified guarantees of robustness.

\subsection{DDPM}
Generally speaking, diffusion models approximate a target distribution $q(x)$ by reversing a gradually Gaussian diffusion. \textit{Denoising diffusion probabilistic models} (DDPMs\cite{ddpm}) define a Markov diffusion process formulated by  
\begin{equation}
q(x_{1:T}|x_0) := \prod_{t=1}^T q(x_t|x_{t-1}), \quad
q(x_t|x_{t-1}):= \mathcal{N}(x_t; \sqrt{1-\beta_t} x_{t-1}, \beta_t \boldsymbol{I})
\end{equation}

where $\{\beta_t \in (0,1)\}_{t=1}^T$. In the diffusion process, $x_t$ can be seen as a mixture of $x_{t-1}$ and a Gaussian noise, which means the data sample $x_0$ gradually loses its distinguishable features and finally becomes a random noise. By define $\alpha_t = 1-\beta_t, \bar{\alpha}_t = \prod_{s=1}^t \alpha_s$, we can get a closed form expression of $q(x_t|x_0)$ like
\begin{equation} \label{diff}
q(x_t|x_0) = \mathcal{N}(x_t;\sqrt{\bar{\alpha}_t} x_0, (1-\bar{\alpha}_t)\boldsymbol{I})
\end{equation} 
And then the model constructs a reverse process by learning to predict a "slightly less noised" $x_{t-1}$ given $x_t$. Let $x_T$ be a latent variable
(usually a random noise), a sample of the target distribution $q(x_0)$ can by got through a Markov process defined as
\begin{equation}
p_\theta(x_{0:T}) := p(x_T)\prod_{t=1}^T p_\theta(x_{t-1}|x_t), \quad p_\theta(x_{t-1}|x_t):= \mathcal{N}(x_{t-1}; \mu_\theta(x_t, t), \Sigma_\theta(x_t, t))
\end{equation}
According to Ho et al.\cite{ddpm}, the mean $\mu_\theta(x_t,t)$ is a neural network parameterized by $\theta$ while the variance $\Sigma_\theta(x_t, t))$
is a set of time-dependent constants.

\subsection{DDIM}
Although diffusion models can generate high-quality samples, they have a critical drawback in generative speed because they require many iterations to
reverse the diffusion process. \cite{ddim} accelerate the generation by changing the Markov process to a non-Markovian one. They rewrite the reverse process with a desired variation $\sigma^2_t$.

\begin{equation}
x_{t-1} =  \sqrt{\bar{\alpha}_{t-1}}(\dfrac{x_t - \sqrt{1-\bar{\alpha}_t} \epsilon_\theta(x_t, t)}{\sqrt{\bar{\alpha}_t}}) + 
\sqrt{1-\bar{\alpha}_{t-1}-\sigma_t^2} \cdot \epsilon_\theta(x_t, t) + \sigma_t \epsilon_t
\end{equation}

Without re-training a DDPM, we can accelerate the inference by only sampling a subset of $S$ diffusion steps $\{\tau_1, \tau_2,\cdots, \tau_S\}$ where $\tau_1<\tau_2<\cdots<\tau_S\in[T]$ and $S<T$.

\begin{equation}
x_{\tau_{i-1}} = \sqrt{\bar{\alpha}_{\tau_{i-1}}}(\dfrac{x_{\tau_i} - \sqrt{1-\bar{\alpha}_{\tau_i}} \epsilon_\theta(x_{\tau_i}, \tau_i)}{\sqrt{\bar{\alpha}_{\tau_i}}}) + 
\sqrt{1-\bar{\alpha}_{\tau_{i-1}}-\sigma_{\tau_i}^2} \cdot \epsilon_\theta(x_{\tau_i}, \tau_i) + \sigma_{\tau_i} \epsilon
\end{equation}

The desired variation $\{\sigma_t^2\}_{t=1}^T$ is controlled by a hyper-parameter $\eta \in \mathbb{R}_{\geq 0}$

\begin{equation}
   \sigma_t(\eta) = \eta\sqrt{ (1-\bar{\alpha}_{t-1})/(1-\bar{\alpha}_{t})}\sqrt{1-\bar{\alpha}_{t}/\bar{\alpha}_{t-1}} 
\end{equation}

Note that the DDPM reverse process is a special case of the DDIM reverse process ($\eta$ = 1). And when $\eta=0$ the denoising processes become deterministic and such a model is named \textit{denoising diffusion implicit model} (DDIM \cite{ddim}).

\section{Methods}

\subsection{Diffusion Purification}\label{guide}
We now present the details of our method. As the diffusion length of the forward process $T$ increases, the adversarial perturbations will be dominated by added Gaussian noises. And the reverse process can be viewed as a purification process to recover clean images. Therefore, the restoring step is fundamental for adversarial robustness. Most importantly, we need to deal with the tension between retaining the semantic content of clean images and filtering the adversarial perturbations.

\begin{figure}[b]
\centering
\includegraphics[width=0.7\linewidth]{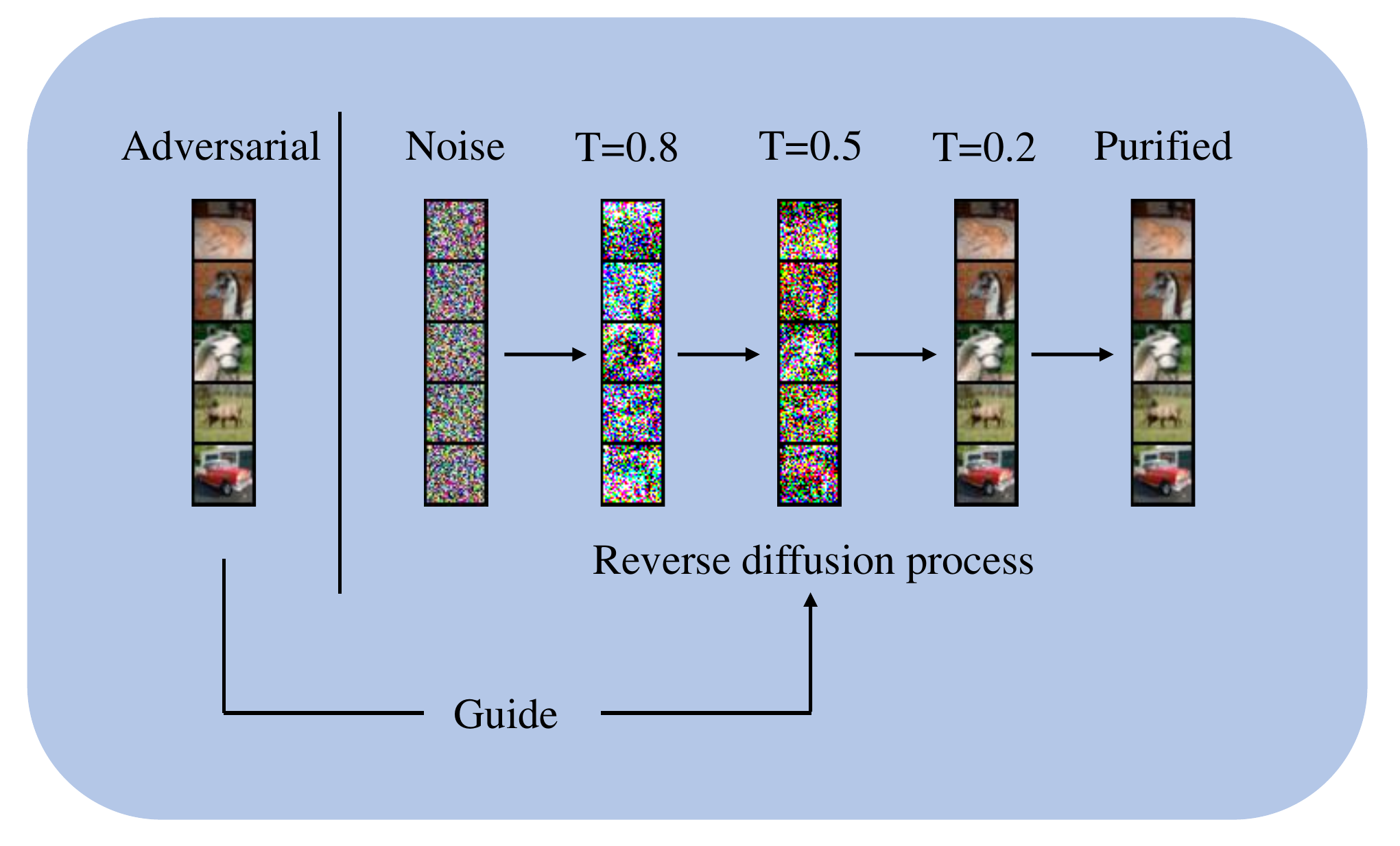}
\caption{An illustration of our diffusion purification methods.}
\label{pipeline}
\end{figure}

Both works \cite{guide, diff_puri} feed the diffused adversarial inputs into the reverse process. \cite{diff_puri} regards the reverse process as solving stochastic differential equations (SDE). In our works, we propose a novel approach inspired by conditional image generation \cite{beatGAN}. We sample the initial input $x_T$ from \textbf{pure Gaussian noises} and gradually denoise 
it with the guidance of the adversarial image $x_0'$. The intuition here is that the diffused image $x_T'$ still carries corrupted structures, the reverse process is likely to get stuck in the local blind spots, which are susceptible to adversarial attacks. On contrary, if the magnitude of the guidance is suitably adjusted, strong enough to recover semantic contents but not too large to resemble the adversarial image, starting from pure Gaussian noises often leads to better results.

Now we explain the guidance mechanism. Let $x_0'$ be the adversarial input, the conditional backward process can be defined as
\begin{equation}
p_{\theta, \phi}(x_{t-1}|x_t, x_0') = Z_1p_\theta (x_{t-1}|x_t)p_\phi(x_0'|x_{t-1})
\end{equation}
where $Z_1$ is a normalizing factor. It has been proved in \cite{beatGAN}
 that we can approximate the distribution using a Gaussian with shifted mean
\begin{equation}
p_\theta (x_{t-1}|x_t)p_\phi(x_0'|x_{t-1}) = \mathcal{N}(\mu+\Sigma \nabla_{x_t} \log p_\phi(x_0'|x_t), \Sigma)
\end{equation}

Following \cite{control, guide}, we use a heuristic formulation here. First we sample the diffused outputs $x_t'$ with the adversarial image $x_0'$ according to \ref{diff}. Suppose $\mathcal{D}$ is a distance metric and $s$ is a scaling factor controlling the magnitude of guidance, we have
\begin{equation}
p_\phi(x_0'|x_t) = Z_2\exp(-s\mathcal{D}(x_t, x_t'))
\end{equation}

Therefore, the mean of the conditional distribution is shifted by $-s\Sigma \nabla_{x_{t}} \mathcal{D}(x_{t}, x_{t}')$. Apparently, introducing an extra neural network, which generates feature embeddings to measure the distances between images, is still vulnerable to white-box attacks. So we adopt the simple Mean Square Error as the distance metric $\mathcal{D}$ in our experiments. As for the guidance scale $s$, we follow the formula used in \cite{guide}. Suppose the adversarial image $x_0' = x_0 + \Delta x$, we have 
\begin{equation}
    x_t' = \sqrt{\bar\alpha_t} x_0 + \sqrt{\bar\alpha_t} \Delta x + \sqrt{1-\bar\alpha_t} \epsilon
\end{equation}
The scaling factor $s$ should be time-dependent, proportional to the ratio of the magnitude of Gaussian noises to adversarial perturbations. A higher ratio suggests that the adversarial perturbation gets dominated by Gaussian noise, thus it's safe to increase the strength of the guidance signal. Suppose the perturbation $\Delta x$ is bounded by $\ell_\infty$ norm $r$, We define $s_t$ using the following formula
\begin{equation}
    s_t = a \cdot \dfrac{\sqrt{1-\bar\alpha_t}}{r \sqrt{\bar\alpha_t}}
\end{equation}
where $a$ is an empirically determined hyperparameter. Above all, the pseudo-code of our method is shown in Algorithm \ref{our_alg}. An illustration of our method is shown in Fig. \ref{pipeline}. 

\begin{algorithm}[tb]
   \caption{Guided Adversarial Purification Using Diffusion Models}
\begin{algorithmic} \label{our_alg}
   \STATE {\bfseries Input:} an adversarial image $x_0'$, diffusion length $T$, distance metric $\mathcal{D}$, scaling factor $s$
   \STATE $x_T \leftarrow$ sample from $\mathcal{N}(0, \boldsymbol{I})$
   \FOR{$t=T$ {\bfseries to} $1$}
   \STATE $x_t' \leftarrow$ sample from $\mathcal{N}( \sqrt{\bar{\alpha}_t} x_0', (1-\bar{\alpha}_t)\boldsymbol{I})$
   \STATE $\mu, \Sigma \leftarrow \mu_\theta(x_t, t), \sigma_t^2\boldsymbol{I}$
   \STATE $x_{t-1} \leftarrow$ sample from $\mathcal{N}(\mu -s\Sigma \nabla_{x_t} \mathcal{D}(x_t, x_t'), \Sigma)$
   \ENDFOR
   \RETURN $x_0$
\end{algorithmic}
\end{algorithm} 

\subsection{Certifying \texorpdfstring{$\ell_2$}. Robustness} \label{robust}
Now we shift our attention to the unguided purification model proposed in \cite{diff_puri}. As pointed out by \cite{diff_puri,guide}, we can purify the adversarial inputs multiple times to gradually eliminate the perturbations. Let $M$ be the number of purification runs, we restate the algorithm as shown in pseudo-code \ref{unguided}.  One main contribution of our works is to explore the inherent connections between \textbf{unguided} diffusion models and randomized smoothing \cite{smooth}. We discover that a provable defense mechanism becomes plausible based on diffusion purification. However, we leave certifying the robustness of the proposed guided diffusion model for future works.

Before diving into the certified $\ell_2$ robustness, we first analyze the properties of the method \ref{unguided} to provide a deeper understanding.
Following the key concepts in \emph{Differential Privacy}, we characterize the robustness of our models using $(\epsilon, \delta)$ notations. 

\begin{Def}[$(\epsilon,\delta)$-Robustness]
	\label{label}
    Let $\mathcal{A}:\mathcal{X} \rightarrow \mathcal{Y}$ be a randomized algorithm, we say $\mathcal{A}$ satisfies $(\epsilon, \delta)$\text{-} Robustness under $\ell_2$ radius $r$, for a training sample $(x,y)$ and $\forall ||\Delta x||_2 < r$, if we have
\begin{equation}
\mathbb{P}(\mathcal{A}(x+\Delta x) \neq y) < e^\epsilon\cdot \mathbb{P}(\mathcal{A}(x) \neq y) + \delta
\end{equation}
\end{Def}

\begin{algorithm}[t]
   \caption{Unguided Adversarial Purification Using Diffusion Models}
\begin{algorithmic} \label{unguided}
   \STATE {\bfseries Input:} an adversarial image $x_0'$, diffusion length $T$, purification run $M$
   \FOR{$i= 1$ {\bfseries to} $M$}
   \STATE $x_T' \leftarrow$ sample from $\mathcal{N}( \sqrt{\bar{\alpha}_T} x_0', (1-\bar{\alpha}_T)\boldsymbol{I})$
   \FOR{$t=T$ {\bfseries to} $1$}
   \STATE $\mu, \Sigma \leftarrow \mu_\theta(x_t', t), \sigma_t^2\boldsymbol{I}$
   \STATE $x_{t-1}' \leftarrow$ sample from $\mathcal{N}(\mu, \Sigma)$
   \ENDFOR
   \ENDFOR
   \RETURN $x_0'$
\end{algorithmic}
\end{algorithm} 

Now we list the notations used in the following arguments. Let $Q: \mathcal{X}  \rightarrow \mathcal{X} $ and $R: \mathcal{X}  \rightarrow \mathcal{X} $ denote the diffusion process and the reverse process. And we have a base classifier $G:\mathcal{X} \rightarrow \mathcal{Y}$, which maps the purified images to the predicted labels. In our framework, the overall randomized algorithm $\mathcal{A}$ can be decomposed as $\mathcal{A} = G\circ R \circ Q$. We use $\Phi$ to denote the standard Gaussian CDF and $\Phi^{-1}$ for its inverse. We also stick to the notations in the last section, i.e. $x_0$ represents the clean image while $x_0'$ denotes the adversarial one.

\begin{theorem}  \label{theorem1}
Our randomized algorithm $\mathcal{A}$ satisfies $(\epsilon, \delta)$\text{-}Robustness under $\ell_2$ radius $r$, where $\delta = \Phi(\frac{r\sqrt{\bar{\alpha}_T}} {2(1-\bar{\alpha}_T)} - \frac{\epsilon}{r\sqrt{\bar{\alpha}_T}})$.
\end{theorem}

Intuitively, a large diffusion length $T$ will help remove adversarial perturbation $\Delta x$. In addition, more purification iterations $M$ lead to a cleaner image after filtering. We have the following theorem

\begin{theorem}
\label{theorem2}
If we keep the radius $r$ and $\epsilon$ fixed, with the increase of diffusion length $T$ and purification iterations $M$, the escape probability $\delta$ will strictly decrease when $\delta > 0$, which provides a stronger robustness guarantee.
\end{theorem}

The detailed proofs of theorem 1 and theorem 2 are presented in the appendix. We can also revisit the concept of robustness through a perspective of KL-divergence. 

\begin{theorem}
\label{theorem3}
    Let $p(x),  q(x)$ denote the probability density function of $\mathcal{A}(x_0)$ and $ A(x_0')$ respectively. $D_{KL}(p(x)||q(x))$ strictly decreases as the diffusion length $T$ and purification run $M$ increases.
\end{theorem}

Please refer to the appendix for the details. Note that the entire diffusion process $Q$ corresponds to adding Gaussian noises to the inputs, which is identical to randomized smoothing \cite{smooth}. And it's safe to regard the reverse process $R$ adjoined with the classifier $G$ as a randomized post-processing step.
Therefore, the certified guarantees in the original paper \cite{smooth} can be generalized directly to our methods. To the best of our knowledge, we first propose a verifiable adversarial defense via diffusion models. Similarly, we now define a smoothed classifier $\tilde{\mathcal{A}}$ with the following formula
\begin{equation}
\tilde{\mathcal{A}}(x) = \arg\max_{c \in \mathcal{Y}} \mathbb{P}(\mathcal{A}(x)=c)
\end{equation}

Now we present the theorem of our $\ell_2$ robustness guarantee: 
\begin{theorem}
Suppose $c_A \in \mathcal{Y}$ and $\underline{p_A}, \overline{p_B} \in [0,1]$ satisfy:
\begin{equation}
\mathbb{P}(\mathcal{A}(x) = c_A) \ge \underline{p_A} \ge \overline{p_B}
\ge \max_{c\not= c_A} \mathbb{P}(\mathcal{A}(x) = c)
\end{equation}
Then $\tilde{\mathcal{A}}(x+\Delta x) = c_A$ ~for all $\|\Delta x\|_2 < R$, where
\begin{equation}
    R = \dfrac{\sqrt{1-\bar{\alpha}_T}}{2}(\Phi^{-1}(\underline{p_A})-\Phi^{-1}(\overline{p_B}))
\end{equation}
\end{theorem}

The proofs are quite straightforward based on the conclusion in \cite{smooth}. Fundamentally, we can interpret the reverse process $R$ plus the classifier $G$ as a powerful  ``noise'' classifier, which generates predictions given the diffused images. However, the reverse process is capable of recovering clean inputs gradually from tricky images, which are close to pure noises and lack evident semantic attributes visually. In other words, our methods are more \textbf{noise-tolerant} and robust compared with randomized smoothing. We can increase the scale of noises without worrying too much about losing semantic information, which helps maintain the accuracy of the base classifier. 

In addition, as the generative models are trained on the natural data, the corrupted images will be pushed towards the real distribution, eliminating the adversarial perturbations. Furthermore, classifiers utilizing randomized smoothing \cite{smooth} require Gaussian data augmentation in training while our diffusion models can be trained independently without compromising performance. Above all, our models are more likely to provide a better certified guarantee. The empirical studies in Sec.4 show that diffusion purification achieves better results than randomized smoothing.

\begin{table}[t]
  \caption{CIFAR-10 results against preprocessor-blind PGD $\ell_\infty$  attacks ($\epsilon = 8/255$).  The diffusion length $T = 1000$ with reverse steps $ 50$.}
  \label{result}
  \centering
  \begin{tabular}{cccc}
    \toprule
    \cmidrule(r){1-2}
    Methods &  Standard Acc   & Robust Acc   & Architecture \\
    \midrule
    Ours & $91.40$ & $89.62$ & ResNet-50 \\
    GDMP(SSIM) \cite{guide} & $93.50$ & $90.10$ & WRN-28-10 \\
    ADP($\sigma=0.1$) \cite{yoon2021adversarial}  & $93.09$ & $85.45$ & WRN-28-10 \\
    \citet{hill2020stochastic} & $84.12$ & $78.91$ & WRN-28-10 \\
    \citet{pgd} & $87.30$ & $70.20$ & ResNet-56\\
    \bottomrule
  \end{tabular}
\end{table}

\section{Experiments}
Our works are based on the official codes of diffusion models \cite{ddpm, improve}. And we implemented the guidance mechanism on our own. Due to the limitation of time, we only conduct several toy experiments using ResNet-50 \cite{resnet} classifier against PGD attacks \cite{pgd}. 

\subsection{Evaluation Results}
\noindent \textbf{Robust Accuracy.} Table \ref{result} shows the robustness  performance against $\ell_\infty$ threat model ($\epsilon = 8/255$) with PGD attack. Our methods utilizing ResNet-50 attain a robust accuracy of $89.62\%$, very close to the results reported in \cite{guide}. We will conduct additional experiments on architectures such as WideResNet and try stronger adaptive attacks like AutoAttack in the future.

\noindent \textbf{Certified Robustness.} In Sec.3.2 we emphasize that our main contribution is to apply the diffusion-model-based adversarial purification to certified robustness. We follow \cite{smooth} to use the binomial hypothesis test to compute the certified radius $R$ for test examples. We set $N_0 = 20, N = 1,000, \alpha = 0.001$ in our experiments. Please refer to the original paper for the definitions. For the fairness of comparison, we recompute the certified radius using the released codes \cite{smooth} in an identical setting. Since we sample fewer points than the original implementation, the lower bound on the certified radius is less tight. If a given example is correctly classified and the certified radius $R$ is larger than $r$, it will be counted to measure the approximate certified test accuracy. The results are shown in Table \ref{certified}. We set the diffusion length to $T=2000$, the corresponding variance is $\sqrt{1-\bar\alpha_T} = 0.59$.  Our methods improve the certified accuracy by about $5\%$ for $r\ge0.5$, which demonstrates the strength of the diffusion-model-based purification.

\begin{table}
  \caption{Approximate certified test accuracy on CIFAR-10 compared with randomized smoothing \cite{smooth}. Each row is a setting of the hyperparameter $\sigma$, each column is an $\ell_2$ radius. We evaluate the models on $10\%$ of the test examples. For comparison, random guessing would attain 0.1 accuracy.}
  \label{certified}
  \centering
  \begin{tabular}{c|ccccc}
    \toprule
    \cmidrule(r){1-2}
     &  $r=0.0$  &  $r=0.25$   &  $r=0.5$   & $r=0.75$ &  $r=1.0$ \\
    \midrule
    $\sigma = 0.25$ &  $\mathbf{0.73}$ & $\mathbf{0.57}$ & $0.37$ & $0.00$ & $0.00$\\
    $\sigma = 0.50$ &  $0.63$ & $0.51$ & $0.40$ & $0.27$ & $0.16$\\
    $\sigma = 1.00$ &  $0.43$ & $0.37$ & $0.31$ & $0.25$ & $0.19$\\
    \midrule
    Ours ($\sigma = 0.59$) & $0.65$ & $0.55$ & $\mathbf{0.44}$ & $\mathbf{0.33}$ & $\boldsymbol{0.24}$\\
    \bottomrule
  \end{tabular}
\end{table}

\subsection{Ablation Studies}
Our ablation studies are based on a diffusion model trained with total steps of $5000$
and under the attack of PGD-$\ell_2$ with a radius of 8/255, some results may be
different from that under PGD-$\ell_{\inf}$.
In order to reserve the major information, our diffusion and generation process only uses the first thousand steps ($T=1000$).

\begin{table}
  \caption{From noise vs.from adversarial image.}
  \label{sample-table}
  \centering
  \begin{tabular}{cccc}
    \toprule
    \cmidrule(r){1-2}
    Methods &  From Noise   &  Guided  & Robust Acc \\
    \midrule
    Baseline & \cmark & \cmark & $91.89$ \\
    (a) & \xmark & \cmark & $83.68$ \\
    (b) & \xmark & \xmark &  $79.54$\\
    \bottomrule
  \end{tabular}
  \label{tab:from_noise}
\end{table}

\noindent \textbf{Starting from random noise.} 
Our first finding is that generating with guidance from random noise can get better results than generating from the adversarial data itself \ref{tab:from_noise}. This really surprised us. We think it is because the diffusion model is enough powerful to
generate high-quality images with resolution 32x32 from random noise and by this way
it can eliminate more adversarial perturbation.

\begin{table}[htbp]
\caption{The randomness experiment results.}
\centering
\begin{tabular}{|c|c|c|c|c|c|}
\hline
$\eta$ & $0$ & $0.1$ & $0.2$ & $0.5$ & $1$\\ \hline
Robust Accuracy & $10.03$ & $11.23$ & $86.34$ & $90.68$& $\mathbf{91.41}$          \\ \hline
\end{tabular}
\label{tab:eta}
\end{table}

\noindent \textbf{The stochastic parameter $\eta$.}
The same with DDIM \cite{ddim}, we use hyper-parameter $\eta \in \mathbb{R}_{\geq 0}$ to control the randomness of generation. When $\eta=0$ the reverse process is deterministic and when $\eta=1$ it is the same with DDPMs \cite{ddpm}. 
The results \ref{tab:eta} show that the robust accuracy drops when $\eta$ decreases since less randomness makes the model more fragile to PGD attacks.

\begin{table}[htbp]
\caption{The acceleration experiment results.}
\centering
\begin{tabular}{|c|c|c|c|c|}
\hline
Respacing steps & $25$ & $50$ & $100$ & $200$ \\ \hline
Acceleration ratio & $40\times$    & $20\times$  & $10\times$  & $5\times$\\ \hline
Robust Accuracy  &    9.17 & 91.41  & 91.65 & $\mathbf{91.89}$          \\ \hline
\end{tabular}
\label{tab:steps}
\end{table}

\noindent \textbf{Time Respacing.}
To accelerate the generation process, we only sample on a subset of the total $T$ steps. The diffusion model needs some steps to generate a sensible image while too many steps take too much time. We consider setting respace steps to $50$ an optimal choice in terms of both speed and accuracy.

\section{Conclusion}
In this project, we propose a novel approach for adversarial purification based on diffusion models. Diffusion models are an ideal candidate for adversarial purification, as the forward process of adding Gaussian noises can be regarded as local smoothing. In the denoising process, diffusion models are capable of recovering clean images, eliminating the Gaussian noises and adversarial perturbations simultaneously. Different from recent works \cite{guide, diff_puri}, we argue that generating from random noises with the guidance of adversarial inputs leads to better purification effects. The experimental results show the advantages of our method.  Furthermore, we first generalize the unguided diffusion purification to certified learning. As diffusion models are more noise-tolerant, we are more likely to obtain a larger certified radius. It achieves higher certified test accuracy compared with randomized smoothing.

However, the limitations of diffusion purification are also obvious. To start with, it takes about $0.92$s to purify an adversarial image of size $32\times 32$ in the CIFAR-10 dataset with reverse step $50$. Meanwhile, the certified smooth classifier requires Monte-Carlo sampling. When we increase the resolution of images and the diffusion length $T$, the inference speed becomes unbearably slow. Perhaps boosting the efficiency of diffusion models can be explored in the future. In our theoretical analysis, we directly apply the unguided diffusion model to certified learning. But our experiments demonstrate that guided purification from random noise outperforms the previous method. Since the guidance involves extra information, it's more difficult to present a tight bound of certified radius.  It might be a future direction as well.

\bibliography{neurips_2022}


\section*{Appendix}
\subsection*{A: Proofs of Theorem 1} \label{thm1}
\begin{theorem*} \ref{theorem1}
    Our randomized algorithm $\mathcal{A}$ satisfies $(\epsilon, \delta)$\text{-}Robustness under $\ell_2$ radius $r$, where $\delta = \Phi(\frac{r\sqrt{\bar{\alpha}_T}} {2(1-\bar{\alpha}_T)} - \frac{\epsilon}{r\sqrt{\bar{\alpha}_T}})$.
\end{theorem*}

\begin{proof}

    Recall that we use $Q$ to denote the forward process. Let $x_T'=Q(x+\Delta x)$ and $x_T=Q(x)$ denote the outputs of the diffusion process with the adversarial image and clean image. Suppose $p(x)$ and $q(x)$ are their probability density functions, respectively. We have
    $$p \sim \mathcal{N}(\sqrt{\bar{\alpha}_{T}} x_0, (1-\bar{\alpha}_{T})\boldsymbol{I_d}), ~q \sim \mathcal{N}(\sqrt{\bar{\alpha}_{T}} (x_0+\Delta x), (1-\bar{\alpha}_{T})\boldsymbol{I_d})$$
    
    Now we use the set $S = \{x:q(x)\geq e^\epsilon \cdot p(x)\}$ to describe the ``bad events''. Equivalently
    $$S = \{x:  [x-\sqrt{\bar{\alpha}_T} (x_0+\Delta x)]^\mathrm{T}\Delta x \geq -\frac{\sqrt{\bar{\alpha}_T} }{2}||\Delta x||_2 ^ 2 + \dfrac{1-\bar{\alpha}_T}{\sqrt{\bar{\alpha}_T}} \epsilon \}$$
    Using the property of Gaussian distribution and the condition that $\|\Delta x\|_2 < r$, we have
    $$\mathbb{P}(Q(x+\Delta x) \in S) \leq \Phi(\dfrac{r\sqrt{\bar{\alpha}_T}} {2(1-\bar{\alpha}_T)} - \dfrac{\epsilon}{r\sqrt{\bar{\alpha}_T}})$$
    For any possible event $E$, We can obtain that
    $$\mathbb{P}(Q(x + \Delta x) \in E) \leq e^\epsilon \cdot \mathbb{P}(Q(x) \in E) + \Phi(\dfrac{r\sqrt{\bar{\alpha}_T}} {2(1-\bar{\alpha}_T)} - \dfrac{\epsilon}{r\sqrt{\bar{\alpha}_T}})$$ 
    Using the post-processing theorem, we have 
    $$\mathbb{P}(\mathcal{A}(x + \Delta x) \in E) \leq e^\epsilon \cdot \mathbb{P}(\mathcal{A}(x) \in E) + \Phi(\dfrac{r\sqrt{\bar{\alpha}_T}} {2(1-\bar{\alpha}_T)} - \dfrac{\epsilon}{r\sqrt{\bar{\alpha}_T}})$$

    Set the event $E$ to $\{\mathcal{A}(x)\not= y\}$, the proof is completed.
\end{proof}

\subsection*{B: Proofs of Theorem 2} \label{thm2}

\begin{theorem*} \ref{theorem2}
If we keep the radius $r$ and $\epsilon$ fixed, with the increase of diffusion length $T$ and purification iterations $M$, the escape probability $\delta$ will strictly decrease when $\delta > 0$, which provides a stronger robustness guarantee.
\end{theorem*}

\begin{proof}
    Using the post-processing theorem, no matter what the reverse process does, the robustness is not going to lose. So we only need to prove robustness increases each step of the forward process.\\
    The forward process is given by $$q(x_t|x_{t-1}) = \mathcal{N}(x_t; \sqrt{1-\beta_t}x_{t-1}, \beta_t \boldsymbol{I})$$
    Denote $p_{t-1}(x)$ as the density function of $\sqrt{1-\beta_t}x_{t-1}$ of the original image and $q_{t-1}(x)$ as the density function of $\sqrt{1-\beta_t}x_{t-1}$ of the adversarial image, $z\prime (x)$ as the density function of $\mathcal{N}(0, \beta_t \boldsymbol{I})$, $p_t(x)$ as the density function of $x_t$ of original image and $q_t(x)$ as the density function of $x_t$ of the adversarial image.\\
    $\forall S \subseteq \mathcal{R}^d$, using convolution theorem, we have
    \begin{align*}
        P(q_t\in S) - e^\epsilon P(p_t\in S) &= \int _S q_t(x) - e^\epsilon p_t(x) \mathrm{d}x\\
        &= \int _S \int _{-\infty}^\infty q_{t-1}(y) z\prime(x-y) - e^\epsilon p_{t-1}(y) z\prime(x-y) \mathrm{d}y \mathrm{d}x\\
        &= \int _{-\infty}^\infty  [q_{t-1}(y) - e^\epsilon p_{t-1}(y)]\int _S z\prime(x-y) \mathrm{d}x\mathrm{d}y\\
        &\leq \int _{-\infty}^\infty [q_{t-1}(y) - e^\epsilon p_{t-1}(y)]\mathbb{I}[q_{t-1}(y) - e^\epsilon p_{t-1}(y) > 0]\mathrm{d}y\\
        &= \int _{q_{t-1}(y) - e^\epsilon p_{t-1}(y) > 0}[q_{t-1}(y) - e^\epsilon p_{t-1}(y)]\mathrm{d}y\\
        &\leq \delta_{t-1}
    \end{align*}
    So if $\epsilon$ is fixed, we have $\delta_t \leq \delta_{t-1}$, where equal is possible only if $\forall y, \mathbb{I}[q_{t-1}(y) - e^\epsilon p_{t-1}(y) > 0] = 0$.
    In conclusion, the robustness will increase monotonically unless $\delta = 0$.
\end{proof}


\subsection*{C: Proofs of Theorem 3} \label{thm3}

\begin{lemma*}
    Let $x(t)$ be the diffusion process defined as $\mathrm{d}x = f(x,t)\mathrm{d}t + g(t)\mathrm{d}w$, where $w$ is standard wiener process. Diffuse two distribution $p(x)$ and $q(x)$ from $p_0(x)$ and $q_0(x)$, we have 
    $$\dfrac{\partial D_{KL}(p_t||q_t)}{\partial t} = -\frac{1}{2} g^2(t) D_F(p_t||q_t)$$
    where $D_F$ is Fisher divergence, $D_F(p_t||q_t) = 0$ if and only if $p_t = q_t$.\\
\end{lemma*}

\begin{proof}
    The Fokker-Planck equation for forward SDE is given by
    \begin{align*}
        \dfrac{\partial p_t(x)}{\partial t} &= -\nabla_x (f(x,t)p_t(x) - \dfrac{1}{2} g^2(t) \nabla_x p_t(x))\\
        &= \nabla_x(h_p(x,t)p_t(x))
    \end{align*}
    where $h_p(x,t) = \frac{1}{2} g^2(t)\nabla_x \log p_t(x) - f(x,t)$.
    
    We can evaluate that
    \begin{align*}
        \dfrac{\partial D_{KL}(p_t||q_t)}{\partial t} &= \dfrac{\partial}{\partial t} \int p_t(x)\log \dfrac{p_t(x)}{q_t(x)}\mathrm{d}t\\
        &= \int \dfrac{\partial p_t(x)}{\partial t} \log \dfrac{p_t(x)}{q_t(x)}\mathrm{d}t - \int \dfrac{p_t(x)}{q_t(x)}\dfrac{\partial q_t(x)}{\partial t}\\
        &= -\int p_t(x) [h_p(x,t) - h_q(x,t)]^T[\nabla_x \log p_t(x) - \nabla_x \log q_t(x)] \mathrm{d}t \\
        &= -\frac{1}{2} g^2(t) D_F(p_t||q_t)
    \end{align*}
\end{proof}

\begin{theorem*} \ref{theorem3}
     Let $p(x),  q(x)$ denote the probability density function of $\mathcal{A}(x_0)$ and $ A(x_0')$ respectively. $D_{KL}(p(x)||q(x))$ strictly decreases as the diffusion length $T$ and purification run $M$ increases.
\end{theorem*}

\begin{proof}
    When $T$ is large enough, we can approximate the forward and reverse process as continuous.\\ 
    The forward process is given by $$q(x_t|x_{t-1}) = \mathcal{N}(x_t; \sqrt{1-\beta_t}x_{t-1}, \beta_t \boldsymbol{I})$$
    The corresponding continuous process is given by
    $$\mathrm{d}x = (\sqrt{1-\beta_t} - 1) x \mathrm{d}t + \sqrt{\beta_t} \mathrm{d}w$$ w is standard wiener process.\\
    The reverse process is given by
    $$p(x_{t-1}|x_t) = \mathcal{N}(x_{t-1}; \frac{1}{\sqrt{\alpha_t}}(x_t - \dfrac{\beta_t}{\sqrt{1-\bar{\alpha}_t}} z_\theta (x_t,t)), \tilde{\beta_t}I_d)$$
    The corresponding continuous process is given by
    $$\mathrm{d}x = -(\frac{1}{2}\beta_t x - \dfrac{\beta_t}{\sqrt{(1 - \beta_t)(1-\bar{\alpha}_t)}}z_\theta (x,t))\mathrm{d}t + \sqrt{\beta_t} \mathrm{d}w$$
    Using the lemma above, we immediately know that the KL term strictly decreases as diffusion length $T$ and purification iterations $M$ increase.
    
\end{proof}

\end{document}